\documentclass{article}
\usepackage{graphicx}
\usepackage{amsmath}
\usepackage{amsfonts}
\usepackage{amssymb}
\usepackage{url}
\usepackage{boxedminipage}
\usepackage{epsf}
\usepackage{algorithm}
\usepackage{algorithmic}
\usepackage{bm}

\usepackage{amsthm}
\usepackage{amsfonts}
\newtheorem{thm}{Theorem}

\newtheorem{lem}[thm]{Lemma}

\numberwithin{equation}{section}

\newcommand{\Real}{\mathbb R}

\def\T{{\mathcal T}}

\def\U{{\mathcal U}}

\def\banditrank{\operatorname{BanditRank}}
\def\alg{\operatorname{Alg}}
\def\E{{\mathbb E}}

\def\pT{{\tilde p}}

\def\R{\mathbb R}

\def\PL{\mathcal{PL}_n}
\def\tourdist{\mathcal{BTL}_n}

\def\D{\mathcal{D}}
\def\OSMDrank{\operatorname{OSMDRank}}
\def\Projection{\operatorname{Projection}}
\def\Decomposition{\operatorname{Decomposition}}

\newcommand{\beats}[3]{[{#1},{#2}]_{#3}}
\newcommand{\myprec}[3]{#1 \prec_{#3} #2}

\newcommand{\myprecC}[4]{#1 \prec_{#4} #2 \prec_{#4}#3}
\newcommand{\myprecS}[4]{#1 \prec_{#4} \genfrac{}{}{0pt}{1}{#2}{#3}}
\newcommand{\myprecSS}[4]{\genfrac{}{}{0pt}{1}{#1}{#2} \prec_{#4} #3}

\newcommand{\cent}[1]{{\hat{#1}}}

\newcommand{\vectheta}{\mbox{\boldmath $\theta$}}

\newcommand{\vecp}{p}
\newcommand{\vecq}{q}
\newcommand{\vecr}{r}

\newcommand{\half}{\frac{1}{2}}

\newcommand{\tildeO}{\Tilde{O}}

\newcommand*{\myfrac}[2]{\genfrac{}{}{0pt}{}{#1}{#2}}

\title{Bandit Online Optimization Over the Permutahedron}

\author{
Nir Ailon
 \and Kohei Hatano \and Eiji
 Takimoto \\
}

\begin{document}

\maketitle

\begin{abstract}
The permutahedron is the convex polytope with vertex set consisting of the vectors $(\pi(1),\dots, \pi(n))$
for all permutations  (bijections) $\pi$ over $\{1,\dots, n\}$.
We study a bandit  game in which, at each step $t$, an adversary chooses a hidden weight weight vector $s_t$,
a player chooses a  vertex $\pi_t$ of the  permutahedron and suffers an observed instantaneous loss of $\sum_{i=1}^n\pi_t(i) s_t(i)$.

We study the problem in two regimes.  In the first regime, $s_t$ is a point in the polytope dual to the permutahedron.
Algorithm CombBand of Cesa-Bianchi et al (2009) guarantees a regret of $O(n\sqrt{T \log n})$ after $T$ steps.  Unfortunately, CombBand requires at each step an $n$-by-$n$ matrix permanent computation,
a $\#P$-hard problem.  Approximating the permanent is possible in the impractical running time of $O(n^{10})$, with an additional heavy inverse-polynomial dependence on the sought accuracy. 
We provide an algorithm of  slightly worse regret  $O(n^{3/2}\sqrt{T})$ but  with  more realistic time complexity $O(n^3)$ per step.
 The technical contribution is a bound on the variance of the Plackett-Luce noisy sorting process's `pseudo loss', obtained by establishing positive semi-definiteness of a family of 3-by-3 matrices of rational functions in exponents of $3$ parameters.

In the second regime, $s_t$ is in the hypercube.  For this case we present and analyze an algorithm based on Bubeck et al.'s (2012) OSMD approach with a novel  projection and decomposition technique for the permutahedron.
The algorithm is efficient and achieves a regret of $O(n\sqrt{T})$, but for a more restricted space of possible loss vectors.

\end{abstract}

\section{Introduction}

Consider a game in which, at each step, a player plays a permutation of some ground set $V=\{1,\dots, n\}$, and then
suffers (and observes) a loss.  We model the loss as a sum over the items of some latent quality of the item, weighted
by its position  in the permutation.  The game is repeated, and the items' quality can adversarially change
 over time.  
The game models many scenarios in which the player is an online system (say, a search/recommendation engine)
presenting a ranked list of items (results/products) to a stream of users.  
A user's experience 
is positive if she perceives the quality of the top items on the list as higher than those at the bottom.
The goal of the system is to create a total positive experience for its users.

There is a myriad of methods for modelling \emph{ranking} loss functions in the literature, especially
(but not exclusively) for information retrieval.
Our choice  allows us to study the problem in the framework of online combinatorial optimization in the \emph{bandit} setting, and to obtain highly nontrivial results improving on state of the art in either run time or regret bounds.  
More formally, we study online linear optimization over the
the \emph{$n$-permutahedron} action set, defined as the convex closure
of all vectors in $\R^n$ consisting of $n$ distinct coordinates taking values in $[n]:=\{1,\dots, n\}$ (permutations).
At each step $t=1,\dots, T$, the player outputs an action $\pi_t$ 
and suffers a loss $ \pi_t' s_t =  \sum_{i=1}^n \pi_t(i)s_t(i)\ ,$
where $s_t\in \R^n$ is the vector of ``item qualities'' chosen by some adversary who knows the player's
strategy but doesn't control their random coins. 
The performance of the player is the difference between their total loss and that of the optimal static player, who plays the best (in hindsight)
single permutation $\pi^*$  throughout.  This difference is known as \emph{regret}.
Note that, given $s_1,\dots, s_T$, $\pi^*$ can be computed by sorting the coordinates of $\sum_{t=1}^T s_t$ in decreasing order.
This is aligned with our practical requirement that items with higher quality should be placed first, and those with lower quality should
be last.

\section{Results, Techniques and Contribution}\label{sec:compare}

Our first of two results, stated as Theorem~\ref{thm:main},  is for the setting in which at each step the loss is uniformly bounded (by $1$ for simplicity) in absolute value for all
possible permutations.  Equivalently, the vectors $s_t$ belong to the polytope that is dual to the  permutahedron.  
Our algorithm, $\banditrank$, plays permutations from a distribution known as the Plackett-Luce model (see \cite{Marden95})
which is widely used in statistics and econometrics (see eg \cite{Beggs19811}). 
It uses an inverse covariance matrix of the distribution in order to obtain an unbiased loss vector estimator, which is a
standard technique \cite{DBLP:journals/jcss/Cesa-BianchiL12}.  The main technical difficulty (Lemma~\ref{lem:main}) is
in bounding second moment properties of Plackett-Luce, by establishing positive semidefiniteness of a certain family of $3$ by $3$ matrices.  The lemma is interesting in its own right as a tool for studying distributions over permutations.
The expected regret of our algorithm is $O(n^{3/2}\sqrt T)$ for $T$ steps, with running time of $O(n^3)$ per time step.
This  result should be compared to CombBand of \cite{DBLP:journals/jcss/Cesa-BianchiL12}, where
a framework for playing bandit games over combinatorially structured sets was developed.  Their techniques
extend that of \cite{DBLP:conf/nips/DaniHK07}.  
In each step, it draws a permutation
from a distribution that assigns to each permutation  $\pi$ a probability of  $e^{\eta \sum_{\tau=1}^t {\pi}' \tilde s_{\tau}}$,
where $\tilde s_t$ is a \emph{pseudo-loss} vector at time $t$, an unbiased estimator of the loss vector $s_t$.
Their algorithm guarantees a 
regret of $O(n\sqrt{T\log n})$, which is better than ours by a factor of $\Theta(\sqrt{n/\log n })$.  However,
its computational requirements are much worse.
In order to draw permutations, they need to compute nonnegative $n$ by $n$ matrix permanents.
Unfortunately, nonnegative permanent computation is $\#P$-hard, as shown by \cite{DBLP:journals/tcs/Valiant79}.
On the other hand, a groundbreaking result of \cite{DBLP:journals/jacm/JerrumSV04} presents a polynomial time approximation scheme
for permanent, which runs in time $O(n^{10})$ for fixed accuracy.
To make things worse, the dependence in the accuracy is inverse polynomial, implying that, even if we could perform arbitrarily accurate floating point
operations, the total running time would be \emph{super linear} in $T$, because a regret dependence of $\sqrt{T}$ over $T$
steps
requires accuracy inverse polynomial in $T$.  (Our algorithm does not suffer from this problem.)
From a practical point of view, the runtime dependence of CombBand in both $n$ and $T$ is infeasible for even modest cases.
For example, our algorithm can handle online ranking of $n=100$ items in an order of few millions of operations per game iteration.   In contrast, approximating the permanent of a $100$-by-$100$ positive matrix is utterly impractical.

We note that independently of our work, Hazan et al. \cite{DBLP:journals/corr/HazanKM13} have improved the state-of-the-art general purpose algorithm for linear bandit optimization, implying an algorithm with regret $O(n\sqrt{T})$ for our problem, but with worse
running time $\tilde O(n^4)$.\footnote{The running time is a product of $\tildeO(n^3)$ number of Markov chain steps required 
for drawing a random point from a convex set under a log-concave distribution, and $O(n\log n)$ time  to test
whether a point lies in the permutahedron.  By $\tilde O$ we hide poly-logarithmic factors.}

In our second result in Section~\ref{sec:OSMD} we further restrict $s_t$ to have $\ell_1$ norm of $1/n$.
(Note that this restriction is contained in $|\pi_t' s_t|\leq 1$ by H{\"o}lder).
We present and analyze an algorithm $\OSMDrank$ based on the bandit algorithm OSMD of
\cite{bubeck-etal:colt12} with  projection and decomposition techniques over
the permutahedron (\cite{yasutake-etal:isaac11,suehiro-etal:alt12}).
The projection is defined in terms of the binary relative entropy
divergence.
The restriction allows us to obtain an  expected regret bound of $O(n\sqrt{T})$
(a $\sqrt{\log n}$ improvement over CombBand).  
The running time is
$O(n^2 + n\tau(n))$, where $\tau(n)$ is the time complexity for some
numerical procedure, which is $O(n^2)$ in a fixed precision machine.


We note previous work on playing the permutahedron online optimization game in the \emph{full information case},
namely, when $s_t$ is known for each $t$. As far as we know, 
Helmbold et al. \cite{helmbold-warmuth:jmlr09}
were
the first to study a more general version of this problem, where the action set is the vertex set of the Birkhoff-von-Neumann
polytope (doubly-stochastic matrices). Suehiro et al. \cite{suehiro-etal:alt12} studied the problem by casting it as a submodularly constrained optimization
problem, giving near optimal regret bounds, and more recently Ailon \cite{ailon14} both provided optimal regret bounds with
improved running time and established tight regret lower bounds.

\section{Definitions and Problem Statement}\label{sec:def}
Let $V$ be a ground set of $n$ items. For simplicity, we identify $V$ with $[n] :=\{1,\dots n\}$.  
Let $S_n$ denote the set of $n!$ permutations over $V$, namely bijections over $[n]$.
By convention, we think of $\pi(v)$ for $v\in V$ as the \emph{position} of $v\in V$ in the ranking, 
where we think of lower numbered positions as \emph{more favorable}.
For distinct $u,v\in V$, we say that $u \prec_\pi v$ if $\pi(u) < \pi(v)$ (in words: $u$ \emph{beats} $v$).  We use $\beats{u}{v}{\pi}$ as shorthand
for the indicator function of the predicate $u \prec_\pi v$.


The convex closure of $S_n$ is known as the permutahedron polytope.  It will be more convenient for us to
consider a translated version of the permutahedron, centered around the origin.  More precisely,
for $\pi\in S_n$ we let $\hat \pi$ denote $$\cent{\pi} := (\pi(1)-(n+1)/2,\, \pi(2)-(n+1)/2,\dots, \pi(n)-(n+1)/2)\ .$$
It will be convenient to define a symmetrized version of the permutation set $\cent{S}_n := \{\cent{\pi}: \pi\in S_n\}$.  The symmetrized $n$-permutahedron, denoted $\cent{P}_n$ is the convex closure of $\cent{S}_n$.  Symmetrization allows us to work with a polytope that is centered around the origin.  Generalization our
result to  standard (un-symmetrized) permutations is a simple technicality that will be explained below.
The notation $u \prec_{\cent{\pi}} v$ and $\beats{u}{v}{\cent{\pi}}$ is defined as for $\pi\in S_n$ in an obvious manner.

At each step $t=1,\dots, T$, an adversary chooses and hides a nonnegative vector $s_t \in \R^n \equiv \R^V$, which assigns an elementwise quality measure $s_t(v)$ for
any $v\in V$.  The player-algorithm chooses a  permutation $\cent{\pi}_t\in \cent{S}_n$, possibly random, and suffers an instantaneous
loss 
\begin{equation}\label{eq:instloss}\ell_t := \cent{\pi}'_t s_t= \sum_{v\in V} \cent{\pi}_t(v)s_t(v)\ .\end{equation}
The total loss $L_t$ is defined as $\sum_{t=1}^T \ell_t$.  We will work with the notion of  \emph{regret},
defined as the difference $L_t-L^*_t$, where
$ L^*_T = \min_{\cent{\pi} \in \cent{S}_n} \sum_{t=1}^T \cent{\pi}' s_t.$
We let $\cent{\pi}^*$ denote any minimizer achieving $L^*_T$ in the RHS.

 For any $\cent{\pi}\in \cent{S_n}$ and $s \in \R^n$,
the dot-product $\cent{\pi}' s$ can be decomposed over pairs:
$\cent{\pi}' s = \frac 1 2\sum_{u\neq v} \beats{u}{v}{\pi}(s(v)-s(u))$.
This makes the symmetrized permutahedron easier to work with.  
Nevertheless, our results also apply to the non-symmetrized permutahedron as well, as we shall see below.

Throughout, the notation $\sum_{u\neq v}$ means summation over distinct, ordered pairs of elements $u,v\in V$, and $\sum_{u<v}$
means summation over distinct, unordered pairs.\footnote{We will only use expressions of the form $\sum_{u<v} f(u,v)$ for symmetric functions satisfying $f(u,v)=f(v,u)$.}
The uniform distribution over $\cent{S}_n$ will be denoted $\U_n$.

The smallest eigenvalue of a PSD matrix $A$ is denoted $\lambda_{\min}(A)$.  The norm $\|\cdot\|_2$ will denote spectral norm (Euclidean norm for a vector).  To avoid notation such as $C,C',C'',C_1$ for universal constants, the  expression $C$ will denote a ``general positive constant'' that may change its value as necessary.  For example, we may write $C=3C+5$.
\section{Algorithm $\banditrank$ and its Guarantee}\label{sec:main}
For this section, we will assume  that the instantaneous losses are uniformly bounded by $1$, in absolute value:
For all $t$ and $\cent{\pi}\in\cent{S}_n$, $|\cent{\pi}'s_t|\leq 1$.    Equivalently, using geometric language, the loss  vectors belong
to a polytope which is \emph{dual} to the permutahedron.

Now consider  Algorithm~\ref{alg:main}.
It maintains, at each time step $t$, a weight vector $w_t \in \R^n$.  At each time step,
it draws a random permutation $\cent{\pi}_t$ from a mixture $\D_t$  of the uniform distribution over $\cent{S_n}$ and 
a distribution $\PL(w)$ which we define shortly.
The distribution mixture is determined
by a parameter $\gamma$.   The algorithm then plays the permutation $\cent{\pi}_t$
and thereby suffers the instantaneous loss defined in (\ref{eq:instloss}).
The weights are consequently updated by adding an unbiased estimator $\tilde s_t$ of $s_t$ (computed using the pseudo-inverse covariance matrix
corresponding to $\D_t$), multiplied by another parameter $\eta>0$.

\paragraph{The Plackett-Luce Random Sorting Procedure:}
The distribution $\PL(w)$ over $\cent{S}_n$, parametrized by $w\in \R^n$, is defined by the following procedure.
  To choose the first (most preferred) item, the procedure draws a random item,
assigning probability proportional to $e^{w(u)}$ for each $u\in V$.  
It then removes this item from the pool of available items, and iteratively
continues to choose the second item, then third and so on. 
As claimed in the introduction, this random permutation model is well studied in statistics.
An important well known property of the distribution is that it can be equivalently defined
as a  \emph{Random Utility Model (RUM)} \cite{Marden95,Yellott77}:  To draw a permutation, add a random iid noise 
variable following the Gumbel distribution to each weight, and then sort the items of $V$ in decreasing value
of noisy-weights.\footnote{The Gumbel distribution, also known as doubly-exponential, has a cdf of $e^{-e^{-x}}$.}
The RUM characterization implies, in particular, that for any two \emph{disjoint} pairs of element $(u,v)$ and $(u',v')$, the events $u \prec_{\pi} v$
and $u' \prec_{\pi} v'$ are statistically independent if $\pi$ is drawn from $\PL(w)$, for any $w$.  This fact will be used later.


\noindent
We are finally ready to state our main result, bounding the expected regret of the algorithm.
\begin{thm}\label{thm:main}
If algorithm $\banditrank$ (Algorithm~\ref{alg:main}) is executed with parameters
$\gamma = O({{n^{3/2}}/\sqrt{T}})$ and $\eta = O(\gamma/n)$,
then the expected regret (with respect to the game defined by the symmetrized permutahedron) is at most
$ O(n^{3/2}\sqrt T)$.  The running time of each iteration is $O(n^3)$. Additionally, there exists
an algorithm with the same  expected regret bound and running time with respect to the standard permutahedron
(assuming the vectors $s_t$ uniformly satisfy $|\pi's_t|\leq 1, \forall \pi\in S_n$.)
\end{thm}

\begin{algorithm}[t!]
\caption{Algorithm $\banditrank(n, \eta, \gamma, T)$\label{alg:main} (assuming $|\cent{\pi}'s_t|\leq 1$ for all $t$ and $\cent{\pi}\in\cent{S}_n$)}
\begin{algorithmic}[1]
        \STATE given: ground set size $n$, positive parameters $\eta, \gamma$ ($\gamma \leq 1$), time horizon $T$
	\STATE set $w_0(u) = 0$ for all $u\in V=[n]$
	\FOR {$t=1..T$}
          \STATE let distribution $\D_t$ over $\cent{S}_n$ denote a mixture of $\U_n$ (with probability $\gamma$)
                          and $\PL(w_{t-1})$ (with probability $1-\gamma$)
          \STATE draw and output $\cent{\pi}_t \sim \D_t$
		\STATE observe and suffer loss $\ell_t$ $(=\cent{\pi}'_t s_t)$  
           \STATE $\tilde s_t =\ell_t  P_t^+ \cent{\pi}_t$  where $P_t=\E_{\cent{\sigma} \sim \D_t}[\cent{\sigma} \cent{\sigma}']$ \label{line:pseudoinverse}
		\STATE set $w_t = w_{t-1} + \eta\tilde s$
	\ENDFOR
\end{algorithmic}
\end{algorithm}

The proof uses a standard technique used e.g. in Cesa-Bianchi et al.'s CombBand \cite{DBLP:journals/jcss/Cesa-BianchiL12}, which is itself
an adaptation of Auer et al.'s Exp3 \cite{Auer:2003:NMB:589343.589365} from the finite case to the 
structured combinatorial case.  The distribution from which the actions  $\cent{\pi}_t$ are drawn in the algorithm differ
from the distribution used in CombBand, and give rise to the technical difficulty of variance estimation, resolved in Lemma~\ref{lem:main}.

\begin{proof}
Let $\T_n$ denote the set of \emph{tournaments} over $[n]$.  More precisely, an element $A\in \T_n$ is
a subset of $[n]\times [n]$ with either $(u,v)\in A$ or $(v,u)\in A$ (but not both) for all $u<v$.  We extend our previous notation
so that $\myprec{u}{v}{A}$ is equivalent to the predicate $(u,v)\in A$.

For any pair $\cent{\pi}\in \cent{S}_n$ and $w\in \R^n$,  $p(\cent{\pi}|w)$ denotes the probability assigned to $\cent{\pi}\in\cent{S}_n$ by
$\PL(w)$.
Slightly abusing notation, we define the following shorthand:
\begin{align}
p(\myprec{u}{v}{}|w) &:= \sum_{\cent{\pi}: \myprec{u}{v}{\cent{\pi}}} p(\pi|w) = \frac{e^{w(u)}}{e^{w(u)}+e^{w(v)}}  \nonumber \\ 
p(\myprecC{u}{v}{z}{}|w) &:= \sum_{\cent{\pi}: \myprecC{u}{v}{z}{\cent{\pi}}} p(\cent{\pi}|w) = \frac{e^{w(u)+w(v)}}{(e^{w(u)}+e^{w(v)}+e^{w(z)})(e^{w(v)}+e^{w(z)})}\ .  \nonumber 
\end{align}

The last two right hand sides 
are easily derived from the definition of the distribution $\PL(w)$,
see also e.g. \cite{Marden95}.  We also define the following abbreviations:
\begin{align}
p(\myprecS{u}{v}{z}{}|w) &:= p(\myprecC{u}{v}{z}{}|w) + p(\myprecC{u}{z}{v}{}|w) 
=  \frac{e^{w(u)}}{e^{w(u)}+e^{w(v)}+e^{w(z)}}  \label{r3} \\
p(\myprecSS{u}{v}{z}{}|w) &:= p(\myprecC{u}{v}{z}{}|w) + p(\myprecC{v}{u}{z}{}|w)   \nonumber \\
&\hspace{0.5cm}= \frac{e^{w(u)+w(v)}}{e^{w(u)}+e^{w(v)}+e^{w(z)}}\left(\frac 1{e^{w(v)}+e^{w(z)}} + \frac 1{e^{w(u)}+e^{w(z)}}\right)  \label{r4} 
\end{align}

We will also need to define a distribution over the set of  tournaments $\T_n$.  The distribution, $\tourdist(w)$ is parametrized by a weight
vector $w\in \R^n$.  Drawing $A\sim \tourdist(w)$ is done by independently setting, for all $u<v$ in $V$,
\begin{align}
(u,v)\in A &\mbox{ with probability } p(\myprec{u}{v}{}|w)=\frac{e^{w(u)}}{e^{w(u)}+e^{w(v)}} \nonumber \\
(v,u)\in A &\mbox{ with probability } p(\myprec{v}{u}{}|w)=\frac{e^{w(v)}}{e^{w(u)}+e^{w(v)}}\ .\nonumber
\end{align}
(Note that the distribution is equivalently defined as the product distribution, over all $u<v$ in $V$, of the Bradley-Terry-Luce 
pairwise preference model, hence the name $\tourdist$.  We refer to \cite{Marden95} for definition and history  of the Bradley-Terry-Luce model.)

For $A\in \T_n$, we denote by $\pT(A|w)$ the probability  $\prod_{\myprec{u}{v}{A}} p(\myprec{u}{v}{}|w)$ of drawing $A$
from $\tourdist(w)$.
The proof of the theorem proceeds roughly as the main result upper bounding the expected regret of CombBand in \cite{DBLP:journals/jcss/Cesa-BianchiL12}.  
The following technical lemma
is required in anticipation of a major hurdle (inequality (\ref{gggg}).
We believe the inequality is interesting in its own right as a probabilistic statement on permutation and tournament distributions.

\begin{lem}\label{lem:main}
Let $s,w \in \R^n$. 
Let  $\cent{\pi} \sim \PL(w)$ and $A\sim \tourdist(w)$ be drawn independently.  Define
$X_1= \sum_{u,v:\ \myprec{u}{v}{\cent{\pi}}}(s(v)-s(u)) =\cent{\pi}'s,\ X_2 = \sum_{u, v:\  \myprec{u}{v}{A}}(s(v)-s(u))$.
Then $\E[X_2^2] \leq \E[X_1^2]$.
\end{lem}
(Note that clearly, $\E[X_2]=\E[X_1]$, so the lemma in fact upper bounds the variance of $X_2$ by that of $X_1$.)
The proof of the lemma is deferred to Section~\ref{sec:prooflem}.

Continuing the proof of Theorem~\ref{thm:main}, we let $q(\pi|w)$ denote the probability of drawing $\pi$ from  the mixture of the uniform 
distribution (with probability $\gamma$) and $\PL(w)$  (with probability $(1-\gamma)$.
  Similarly to above, $q(\myprec{u}{v}{}|w)$ denotes $\sum_{\cent{\pi}: \myprec{u}{v}{\cent{\pi}}} q(\cent{\pi}|w)$.
By these definitions,
\begin{equation} 
q(\cent{\pi}|w) = (1-\gamma) p(\cent{\pi}|w) + \frac \gamma{n!}\ \ \ \ \ \ 
q(\myprec{u}{v}{}|w) = (1-\gamma) p(\myprec{u}{v}{}|w) + \frac \gamma 2 \ . \label{weryreg}\end{equation}

\noindent
The analysis proceeds by defining a potential function:
$W_t(u,v) :=  e^{\frac 1 2\eta(w_{t}(u) - w_{t}(v))} +  e^{\frac 1 2\eta(w_t(v) - w_t(u))}$.
  The quanatity of interest will be $\E\left[\sum_{u<v}\sum_t \log \frac{W_t(u,v)}{W_{t-1}(u,v)}\right]$, where the expectation is taken over all random coins used by the algorithm throughout $T$ steps.  This quantity will be bounded from above and from below, giving rise to a bound on the expected  total loss, expressed using the optimal static loss.
On the one hand,
\begin{align}
\sum_{u<v}& \log \frac{W_t(u,v)}{W_{t-1}(u,v)} = \sum_{u<v} \log \left ( \frac{e^{\frac 1 2(w_t(u)-w_t(v))}}{W_{t-1}(u,v)} + \frac{e^{\frac 1 2(w_t(v)-w_t(u))}}{W_{t-1}(u,v)} \right ) \nonumber \\
&= \sum_{u<v} \log \left ( \frac{e^{\frac 1 2(w_{t-1}(u)-w_{t-1}(v))}e^{\frac 1 2 \eta(\tilde s_t(u)-\tilde s_t(v))}}{W_{t-1}(u,v)} + \frac{e^{\frac 1 2(w_t(v)-w_t(u))} e^{\frac 1 2\eta  (\tilde s_t(v)-\tilde s_t(u))} }{W_{t-1}(u,v)} \right ) \nonumber \\
&= \sum_{u<v} \log \left ( p(\myprec{u}{v}{}|w_{t-1})e^{\frac 1 2 \eta(\tilde s_t(u)-\tilde s_t(v))} + p(\myprec{v}{u}{}|w_{t-1})e^{\frac 1 2\eta  (\tilde s_t(v)-\tilde s_t(u))}\right )\nonumber \\
&= \log \left( \sum_{A \in \T_n}\pT(A|w_{t-1})e^{\frac  1 2 \eta\sum_{\myprec{u}{v}{A}}(\tilde s_t(u)- \tilde s_t(v))} \right )\ . \nonumber 
\end{align}
We will now assume that $\eta$ is small enough so that for all $A\in \T_n$ and for all $t$,
\begin{equation}\label{eq:etaineq}
\eta \left |  \sum_{(u,v)\in A}(\tilde s_t(u)- \tilde s_t(v))\right | \leq 1\ .
\end{equation}
\noindent
(This 
will be shortly enforced.)  Using $e^x \leq 1+x+x^2$  $\forall x\in[-1/2,1/2]$,
\begin{align}
\sum_{u,v}& \log \frac{W_t(u,v)}{W_{t-1}(u,v)} \leq \log \left[ \sum_{A \in\T_n}\pT(A|w_{t-1})
\left (1+\frac  \eta 2 \sum_{\myprec{u}{v}{A}}(\tilde s_t(u)- \tilde s_t(v)) \right . \right . \nonumber \\
&\hspace{6cm}\left . \left .+\frac  {\eta^2} 4 \left (\sum_{\myprec{u}{v}{A}}(\tilde s_t(u)- \tilde s_t(v))\right )^2\right ) \right ]\nonumber  \\\
&= 
 \log \left[ 1 + \frac  \eta 2 \E_{A\sim\tourdist(w_{t-1})}\left [\sum_{\myprec{u}{v}{A}}(\tilde s_t(u)- \tilde s_t(v))
+ \frac  {\eta^2} 4\left(\sum_{\myprec{u}{v}{A}}(\tilde s_t(u)- \tilde s_t(v))\right)^2\right ]\right] \nonumber \\
&\leq \log \left[ 1 + \frac  \eta 2 \E_{\cent{\pi}\sim\PL(w_{t-1})}\left [\sum_{\myprec{u}{v}{\cent{\pi}}}(\tilde s_t(u)- \tilde s_t(v))
+ \frac  {\eta^2} 4\left(\sum_{\myprec{u}{v}{\cent{\pi}}}(\tilde s_t(u)- \tilde s_t(v))\right)^2\right ]\right]\ . \label{gggg}
\end{align}
where we used Lemma~\ref{lem:main} in the last inequality (together with the fact that  the marginal
probability of the event ``$\myprec{u}{v}{Y}$'' is identical for both $Y\sim \PL(w_{t-1})$ and $Y\sim \tourdist(w_{t-1})$).
Henceforth, for any $\cent{\pi}\in \cent{S}$, we let 
$ \tilde \ell_t(\cent{\pi}) := \cent{\pi}' \tilde s_t = \sum_{\myprec{u}{v}{\cent{\pi}}} (\tilde s(v)-\tilde s(u))$.
Using \ref{weryreg} and the fact that $\log(1+x)\leq x$ for all $x$,  we get
\begin{align}
\sum_{u <v}& \log \frac{W_t(u,v)}{W_{t-1}(u,v)} \nonumber \\
&\leq  \frac \eta 2 \sum_{u\neq v} \frac{q(u\prec v|w_{t-1})-\frac \gamma 2}{1-\gamma}(\tilde s_t(u)-\tilde s_t(v))
+ \frac {\eta^2} 4  \sum_{\cent{\pi} \in \cent{S}_n}\frac{q(\pi|w_{t-1}) - \frac \gamma{n!}}{1-\gamma}{\tilde \ell_t(\cent{\pi})}^2 \nonumber\\
&\leq    \frac {-\eta} {2(1-\gamma)}\sum_{\cent{\pi}\in \cent{S}_n}q_t(\cent{\pi}|w_{t-1}) \tilde \ell_t(\cent{\pi}) +  \frac {\eta^2} {{4(1-\gamma)}} \sum_{\cent{\pi} \in \cent{S}_n}{q_t(\cent{\pi}|w_{t-1})} \tilde \ell_t(\cent{\pi})^2 \ .\nonumber
\end{align}

We now note that 
(1)
$\sum_{\cent{\pi}\in\cent{S}}q_t(\cent{\pi}|w_{t-1})\tilde \ell_t = \ell_t$ (following the properties of matrix pseudo-inverse in Line~\ref{line:pseudoinverse} in Algorithm~\ref{alg:main}), and 
(2) 
 $\sum_{\cent{\pi} \in \cent{S}_n}{q_t(\cent{\pi}|w_{t-1})} \tilde \ell_t(\pi)^2] \leq n$ (see top of page 31 together with Lemma~{15}  in \cite{DBLP:journals/jcss/Cesa-BianchiL12}).
 Applying these inequalities, and then taking expectations over the algorithm's randomness and summing for
$t=1,\dots, T$, we get
\begin{align*}
\sum_{t=1}^T\E\left [\sum_{u,v}\log \frac{W_t(u,v)}{W_{t-1}(u,v)}\right] 
\leq   - \frac \eta {2(1-\gamma)}\E[L_T] +  \frac {\eta^2} {8(1-\gamma)}  n  T \ .\\
\end{align*}



\noindent
On the other hand, 
\begin{align*}
\sum_{t=1}^T&\E\left [\sum_{u,v} \log \frac{W_t(u,v)}{W_{t-1}(u,v)}\right ] \\
& \geq \sum_{u,v} \E\left[\log \left(\beats{u}{v}{\pi^*}e^{\frac 1 2 (w_T(u)-w_T(v))} + \beats{v}{u}{\pi^*}e^{\frac 1 2 (w_T(u)-w_T(v))})\right )\right ] 
  - \sum_{u,v}  \log 2 \\
&= \frac 1 2\sum_{u,v}   \left(\E\left [\beats{u}{v}{\pi^*}(w_T(u)-w_T(v))  +  \beats{v}{u}{\pi^*}(w_T(u)-w_T(v))\right]\right) - {n \choose  2}  \log 2 \\
&=\frac \eta 2 \sum_{u,v} \left (\E\left [\beats{u}{v}{\pi^*}\sum_t(\tilde s_t(u)-\tilde s_t(v))  +  \beats{v}{u}{\pi^*}\sum_t(\tilde s_t(u)-\tilde s_t(v))\right ]\right) 
- {n \choose  2}  \log 2 \\
&= \frac \eta 2 \sum_{u,v} \left(\beats{u}{v}{\pi^*}\sum_t( s_t(u)- s_t(v))  +  \beats{v}{u}{\pi^*}\sum_t(s_t(u)- s_t(v))\right) 
- {n \choose  2}  \log 2 \\
&=- \frac \eta 2 L_T^*
- {n \choose  2}  \log 2 \ ,
\end{align*}

\noindent
where $L_T^*$ is the total loss of a player who chooses the best permutatation $\cent{\pi}^*\in \cent{S}_n$ in hindsight.
Combining, we obtain
$\frac {\eta}{2(1-\gamma)}\E[ L_t] \leq \frac \eta 2 L_T^* + \frac{n^2}{2}  \log 2  +  \frac {\eta^2} {4(1-\gamma)}  n  T$.
Multiplying both sides by $2(1-\gamma)/\eta$ yields
\begin{eqnarray}\label{almost}
\E [L_T] \leq  L_T^* + \gamma|L_T^*| + \frac{n^2\log 2}{\eta} + \frac{\eta} 2 nT\ .
\end{eqnarray}

\noindent
We shall now work to impose (\ref{eq:etaineq}).
\begin{eqnarray*}\label{condition}
\max_t \max_{A\in \T(V)} \left |\sum_{(u,v)\in A} (\tilde s_t(u) - \tilde s_t(v)) \right|  
\leq \max_t \sqrt{\sum_{v\in V} \tilde s_t(v)^2}\sqrt{\sum_{i=-(n-1)/2}^{(n-1)/2} i^2}  
\leq  C\max_t \|\tilde s_t\|_2 n^{3/2} \ ,
\end{eqnarray*}
where the left inequality is Cauchy-Schwartz.  We now note that  $\|\tilde s_t\|_2 \leq |\ell_t| \|P^+_t\|_2 \|\cent{\pi}_t\|_2$.  Clearly $\|\cent{\pi}\|_2$ is bounded above by $Cn^{3/2}$.  Also $\|P^+_t\|_2$ equals
$1/\lambda_{\min}(P_t)$.  By Weyl's inequality $\lambda_{\min}(P_t) \geq \gamma \lambda_{\min}(\E_{\cent{\tau}\sim \U_n}[\cent{\tau}\cent{\tau}'])$.
It is an exercise to check that $\lambda_{\min}(\E_{\cent{\tau}\sim \U_n}[\cent{\tau}\cent{\tau}']) \geq Cn^2$.
We conclude (also recalling that $|\ell_t|\leq 1$) that $\max_t \|\tilde s_t\|_2 \leq C/(n^{1/2}\gamma)$.
Combining, we shall satisfy (\ref{condition}) by imposing $\eta \leq \gamma/(Cn)$.  Plugging  in (\ref{almost}), we get
\begin{eqnarray}\label{almost2}
\E [L_T(\alg)] \leq  L_T^*  + \gamma|L_T^*| +   \frac{C  n^3}{\gamma} + {{C\gamma T} }\ .
\end{eqnarray}
\noindent
Choosing $\gamma = \sqrt{\frac{C n^3}{T}}$ gives
$
\E [L_T(\alg)] \leq  L_T^* +  {\frac{ C n^{3/2} }{\sqrt{T}}}|L_T^*| +   n^{3/2}\sqrt{T}
$.

This concludes the required result for the symmetrized case, because $|L_T^*| \leq T$.
For the standard permutahedron, we notice that for any $\pi\in S_n$ and its symmetrized counterpart $\cent{\pi}\in\cent{S_n}$, and any vector $s\in \R^n$, $\pi's - \cent{\pi}'s = \frac{n-1}2\sum_{v\in V} s(v) =: f(s)$.  Equivalently, we can write
$\pi's = (\cent{\pi}', 1)(s; f(s))$, where $(\cdot,a)$ appends the scalar $a$ to the right of a row vector and $(\cdot;a)$
appends  to the bottom of a column vector.   Algorithm~\ref{alg:main} can be easily adjusted to work
with action set $\cent{S}_n\times\{1\}$.  For the proof, we keep the same potential function.
 The  technical
part of the proof is lower bounding the smallest eigenvalue of the expectation of $\cent{\tau}\cent{\tau}'$, where
$\cent{\tau}$ is now drawn from the uniform distribution on $\cent{S}_n\times\{1\}$.
We omit these simple details for lack of space.
\qed
\end{proof}

\subsection{Proof of Lemma~\ref{lem:main}}\label{sec:prooflem}
The expression $\E[X_1^2]$ can be written as
\begin{align}
 \E[X_1^2] &= \sum_{u\neq v} p(u\prec v|w) ((s(v)-s(u))^2  \nonumber \\
&\hspace{0.1cm} +\sum_{ |\{u,v,u',v'\}|=4} p(u\prec v\wedge u'\prec v'|w)\,(s(v)-s(u))(s(v')-s(u')) \nonumber \\
&\hspace{0.1cm} +\sum_{\myfrac{u\neq v,u'\neq v'}{|\{u,v,u',v'\}|=3}} p(u\prec v\wedge u'\prec v'|w)\,(s(v)-s(u))(s(v')-s(u'))   \label{X1}\ ,
\end{align}
where $p(u\prec v\wedge u'\prec v'|w)$ is the probability that both ${u}\prec_{{\cent{\pi}}}{v}$ and  ${u'}\prec_{{\cent{\pi}}}{v'}$ with $\cent{\pi}\sim \PL(w)$.  Similarly,
\begin{align}
 \E[X_2^2] &= \sum_{u\neq v} p(u,v|w) ((s(v)-s(u))^2  \nonumber \\
&\hspace{0cm} +\sum_{ |\{u,v,u',v'\}|=4} p(u\prec v|w)p(u'\prec v'|w)\,(s(v)-s(u))(s(v')-s(u')) \nonumber \\
&\hspace{0cm} +\sum_{\myfrac{u\neq v,u'\neq v'}{|\{u,v,u',v'\}|=3}} p(u\prec v|w)p(u'\prec v'|w)\,(s(v)-s(u))(s(v')-s(u')) \label{X2} \ .
\end{align}

Since  Plackett-Luce is a random utility model (see \cite{Marden95}), it is clear that whenever a pair of pairs $u\neq v,u'\neq v'$ satisfies  $|
\{u,v,u',v'\}|=4$,
 $p(u\prec v\wedge u'\prec v'|w) = p(u\prec  v|w)p(u'\prec v'|w)$.
Hence, it suffices to prove that the third summand in the RHS of (\ref{X2}) is upper bounded by the third summand in the RHS of (\ref{X1}).
But now notice the  following identity:

$$ \sum_{\myfrac{u\neq v, u'\neq v'}{|\{u,v,u',v'\}|=3}} \equiv \sum_{\myfrac{\Delta\subseteq V}{|\Delta|=3}} \sum_{\myfrac{\myfrac{u\neq v, u'\neq v'}{ u,v,u',v'\in \Delta}}{|\{u,v,u',v'\}|=3}}\ .$$
This last sum rearrangement implies that it suffices to prove that for any $\Delta$ of cardinality $3$,
\begin{align}
F_2(\Delta) :=  &\sum_{\myfrac{\myfrac{u\neq v, u'\neq v'}{ u,v,u',v'\in \Delta}}{|\{u,v,u',v'\}|=3}} p(u,v|w)p(u',v'|w)\,(s(v)-s(u))(s(v')-s(u'))  \nonumber \\
& \leq   \sum_{\myfrac{\myfrac{u\neq v, u'\neq v'}{ u,v,u',v'\in \Delta}}{||\{u,v,u',v'\}|=3}}  p(u,v\wedge u',v'|w)\,(s(v)-s(u))(s(v')-s(u')) =: F_1(\Delta) \ . \nonumber
\end{align}

If we now denote $\Delta=\{a,b,c\}$, then both $F_1(\Delta)$ and $F_2(\Delta)$ are quadratic forms in $s(a),s(b),s(c)$ (for fixed $w$).
 It hence
suffices to prove that $H(\Delta):= F_1(\Delta)-F_2(\Delta)$ is a positive semi-definite form in $s(\Delta) := (s(a), s(b), s(c))'$.
We now write
$$ H(\Delta) = s(\Delta)'\left (\begin{matrix} H_{aa} & \frac 1 2 H_{ab} & \frac 1 2 H_{ac}\\ \frac 1 2 H_{ab} & H_{bb} & \frac 1 2 H_{bc} \\
\frac 1 2 H_{ac} & \frac 1 2 H_{bc} & H_{cc} \\ \end{matrix} \right) s(\Delta)\ .$$
The matrix is singular, because clearly $H(\Delta)=F_1(\Delta)=F_2(\Delta)=0$ whenever $s(a)=s(b)=s(c)$. 
 To prove positive semi-definiteness, by Sylvester's criterion it hence suffices to show that the diagonal element   $H_{aa} \geq 0$ and that the principal $2$-by-$2$ minor determinant $H_{aa}H_{bb}- \frac 1 4 H^2_{ab} \geq 0$.
Using the definitions, together with the properties of $\PL(w)$, a technical (but quite tedious) algebraic derivation 
(see Appendix~\ref{sec:Hdetails} for details) gives
\begin{equation}\label{h1} 
H_{aa} = \frac{4e^{s(a)+s(b)+s(c)}}{(e^{s(a)}+e^{s(b)})(e^{s(a)}+e^{s(c)})(e^{s(a)}+e^{s(b)}+e^{s(c)})}\ .\end{equation}
Similarly, by symmetry,
$H_{bb} = \frac{4e^{s(a)+s(b)+s(c)}}{(e^{s(b)}+e^{s(a)})(e^{s(b)}+e^{s(c)})(e^{s(a)}+e^{s(b)}+e^{s(c)})}$. 
From a  similar (yet more tedious) technical algebraic calculation which we omit, one gets:
(see Appendix~\ref{sec:Hdetails} for details):
\begin{equation}\label{h3} H_{ab} = \frac{-8e^{s(a)+s(b)+2s(c)}}{(e^{s(a)}+e^{s(b)})(e^{s(a)}+e^{s(c)})(e^{s(b)}+e^{s(c)})(e^{s(a)}+e^{s(b)}+e^{s(c)})}\ . \end{equation}

\noindent
One now verifies, using (\ref{h1})-(\ref{h3}), the identity
$$ H_{aa}H_{bb} - \frac 1 4 H^2_{ab} = \frac{16e^{2s(a)+2s(b)+2s(c)}}{(e^{s(a)}+e^{s(b)})^2(e^{s(a)}+e^{s(c)})(e^{s(b)}+e^{s(c)})(e^{s(a)}+e^{s(b)}+e^{s(c)})^2}\ .$$

It remains to notice, trivially, that $H_{aa}\geq 0$ and $H_{aa}H_{bb} - \frac 1 4 H_{ab}^2 \geq 0$ for all possible values of $s(a),s(b), s(c)$.  The proof of the lemma is concluded.



\section{Bandit Algorithm based on Projection and Decomposition}
In this section, we propose another bandit algorithm $\OSMDrank$,
described in Algorithm \ref{alg:osmd}.
We will be 
working under the more restricted assumption that $\sup \|s_t\|_1 \leq 1$ and $\sup \|{\cent{\pi}}_t\|_\infty \leq 1$.
This in particular implies that $|\cent{\pi}_t' s_t|\leq 1$, as before.
But now we shall achieve a better expected regret of  $O(n\sqrt{T} )$.
\label{sec:OSMD}
\begin{algorithm}[t]
\caption{Algorithm $\OSMDrank(n, \eta, \gamma, T)$ (assuming
 $\|s_t\|_1 \leq 1$ and $\cent{\pi_t}\in\cent{Q}_n$  for all $t$ )}
\label{alg:osmd}
\begin{algorithmic}[1]
\STATE given: ground set size $n$, positive parameters $\eta, \gamma$ ($\gamma \leq 1$), time horizon $T$
\STATE let $x_1=0  \in \cent{Q}_n$. (Note that $x_1=\arg\min_{a \in \cent{Q}_n}F(a)$)
\FOR {$t=1,\dots,T$}
\STATE  let  $\tilde{x}_t=(1-\gamma)x_t$ (Note that $\tilde{a}_t \in
 \cent{Q}_n$ since the origin $0$ and $x_t$ are in $\cent{Q}_n$  and
 $\tilde{x}_t$ is a convex combination of them).
\STATE output $\pi_t=\Decomposition(\tilde{x}_t)$ (i.e., choose $\pi_t$
 so that $\E[\pi_t]=\tilde{x}_t$) and suffer loss $\ell_t$ ($=\pi_t' s_t$) 
\STATE let distribution $\D_t$ over $[-1,1]^n $denote a mixture of the
 uniform distribution over the canonical basis with random sign (with
 probability $\gamma$) and a Radmacher distribution over $\{-1,1\}^n$
 with parameter $(1+x_ {t,i})/2$  for each $i=1,\dots,n$ (with
 probability $1-\gamma$)
\STATE estimate the loss vector $\tilde{s}_t=\ell_t P_t^+\pi_t$, where
 $P_t=\E_{\sigma \sim D_t}[\sigma\sigma']$ 
\STATE let $x_{t+\half}=\nabla F^*( F(x_t) -\eta \tilde{s}_t)$ 
\STATE let $x_{t+1}=\Projection(x_{t+\half})$ (that is, $x_{t+1}=\min_{x \in \cent{Q}_n}D_F(x, x_{t+\half}))$
\ENDFOR
 \end{algorithmic}
\end{algorithm}

We prefer, for reasons clarified shortly, to require that the actions $\cent{\pi}_t$ are vertices of the
rescaling $\cent{Q}_n := \frac {2}{n-1} \cent{P}_n \in [-1,1]^n$ of the symmetrized
permutahedron.  That is, $\sup \|\cent{\pi}_t\|_\infty\leq 1$ (and $\sup \|s_t\|_1 \leq 1$).
This will allow us to work with the following standard regularizer $F:[-1,1]^n\to \R^+$: 
$F(x)=\frac{1}{2}\sum_{i=1}^n \left(
(1+x)\ln (1+x) +(1-x)\ln (1-x)
\right)$.
The regularizer $F(x)$  is the key to the OSMD (Online Stochastic Mirror
Descent) algorithm of Bubeck et al.~\cite{bubeck-etal:colt12}, on which our algorithm is based.
OSMD is a bandit algorithm over the hypercube domain $[-1,1]^n$ and a
variant of Follow the Regularized Leader (FTRL, e.g.,
\cite{hazan:bookchap11}) for linear loss functions. 
To apply this algorithm, we need a new projection and
decomposition technique for the polytope  $\cent{Q}_n$, as well as a
slightly modified perturbation step in line $4$ of Algorithm \ref{alg:osmd}.
Our algorithm $\OSMDrank$ has the following two procedures:
\begin{enumerate}
\item
{\bf Projection: } Given a point $x_t \in [-1,1]^n$, return
$\arg\min_{y_t \in \cent{Q}_n }\Delta_F(y_t,x_t)$, where $\Delta_F$ is the
Bregman divergence defined wr.t. $F$, i.e., 
$\Delta_F(y,x)=F(y)-F(x)-\nabla F(x)'(y-x)$ (also known as \emph{binary
     relative entropy}).\footnote{Note that the binary relative entropy is
     different from the relative entropy, where the relative entropy is
     defined as 
     $Rel(p,q)=\sum_{i=1}^n p_i \ln \frac{p_i}{q_i}$for probability
     distributions $p$ and $q$ over $[n]$.}
\item
{\bf Decomposition: } Given $y_t \in \cent{Q}_n$ from the the projection step, output  a random vertex $\cent{\pi}_t$ of $\cent{Q}_n$  such that $\E[\cent{\pi}_t]=y_t$. 
\end{enumerate}

The decomposition can be done using the technique of
\cite{yasutake-etal:isaac11}, which runs in $O(n\log n)$ time.  (To be precise, the method there was defined for the
standard permutahedron; The adjustments for the symmetrized version are trivial.)
For notational purposes, we define $f := \nabla F$, and notice that
$f(x)_i = \frac{1}{2}\ln \frac{1+x_i}{1-x_i}$,
and its inverse function $f^{-1}$ is given by
$f^{-1}(y)_i = \frac{e^{y_i}-1}{e^{y_i}+1}$.
Our projection procedure is presented in Algorithm~\ref{alg:proj}.
\begin{lem} \label{lem:projection}
(i) Given $q\in [-1,1]^n$, Algorithm~\ref{alg:proj} outputs the projection of $q$ onto $\cent{Q}_n$, with respect to the
regularizer $F$.
(ii) The time complexity of the algorithm is $O(n\tau(n) + n^2)$,
where $\tau(n)$ is the time complexity to perform step 4.
\end{lem}
\begin{proof}[skecth]
Our projection algorithm is an extension of that in \cite{suehiro-etal:alt12}
 and our proof follows a similar  argument in \cite{suehiro-etal:alt12}.
For simplicity, we assume that elements in $q$ are sorted in descending order, 
i.e., $q_1\geq q_2 \geq \dots \geq q_n$.
This can be achieved in time $O(n\log n)$ by sorting $q$. 
 Then, it can be shown that projection preserves the order in $q$ by
 using Lemma 1 in \cite{suehiro-etal:alt12}.
That is, the projection $p$ of $q$ satisfies $p_1 \geq p_2 \geq \dots \geq p_n$. 
So, if the conditions 
$\frac{2}{n-1}\sum_{j=1}^i p_j \leq \sum_{j=1}^i (\frac{n+1}{2}-j)$, for $i=1,\dots,n-1$, 
are satisfied, then other inequality constraints are satisfied as well since
for any $S \subset [n]$ such that $|S|=i$,
$\sum_{j \in S}p_j \leq \sum_{j=1}^i p_j$.
Therefore, relevant constraints for projection onto $\cent{Q_n}$ are only
 linearly many. 

By following a similar argument in \cite{suehiro-etal:alt12}, we can show that 
the output $\vecp$ indeed satisfies the KKT optimality conditions for projection, which completes the proof of the first statement.
Finally, the algorithm terminates in time $O(n\tau(n) +n^2 )$ since the
 number of iteration is at most $n$ and each iteration takes $O(n + \tau(n))$ time, 
which completes the second statement of the lemma. 
\qed
\end{proof}
Note that with respect to other regularizers (e.g. relative entropy or Euclidean norm squared), 
a different projection scheme is possible in time $O(n^2)$ (see
\cite{yasutake-etal:isaac11,suehiro-etal:alt12} for the details).   It is an open question
whether an $O(n^2)$ algorithm can be devised with respect to the binary relative entropy we need here.
In  our case, we need to solve a numerical optimization problem by, say, binary search. 
Note that the time $\tau(n)$ is reasonably small: In fact, we can perform
the binary search over the domain $[-1,1]$ for each dimension $i$. Therefore,
if the precision is a fixed constant, the binary search ends in time $O(n)$
for each dimension. In that case, $\tau(n)$ is $O(n^2)$.  
We are ready to present our main result for this section.

\begin{thm}
\label{thm:osmd}
For $\eta=O(n\sqrt{1/T})$ and $\gamma=O(\sqrt{1/T})$, 
Algorithm $\OSMDrank$ has expected regret $O(n\sqrt{T})$ and running
 time  $O(n^2 +n\tau(n))$ per step,
where $\tau(n)$ is the time for a numerical optimization step depending on $n$.
Additionally, there exists
an algorithm with the same  expected regret bound and running time with respect to the standard permutahedron
(assuming  $\|s_t\|_1 \leq 1/n$).
\end{thm}

\begin{proof}[sketch]
The algorithm $\OSMDrank$ is a modification of OSMD for the hypercube
 $[-1,1]^n$ obtained by adding (1) a projection step and  (2) a decomposition step.
Standard techniques show that adding the projection step does not increase
 the  expected regret bound (see, e.g., chapters 5 and 7 on OMD and OSMD of Bubeck's lecture notes \cite{bubeck:lec11}).
The key facts are:
(i) A variant of Theorem 2 of \cite{bubeck-etal:colt12} (regret bound of
 OSMD) holds for OSMD with $\Projection$, 
(ii)  $E[\pi_t] = (1-\gamma)x_t$, and
(iii)  The estimated loss is the same one used in OSMD for the hypercube $[-1,1]^n$ .
 Once these three conditions are satisfied,
we can prove a regret bound of $\OSMDrank$ by
 following the proof of Theorem 5 in Bubeck et al. \cite{bubeck-etal:colt12}.
In addition, the running time of OSMD per trial is $O(n)$ 
\cite{bubeck-etal:colt12}.
Combining Lemma~\ref{lem:projection} for the projection and 
the analysis of the decomposition from \cite{yasutake-etal:isaac11}, the
 proof of the first statement is concluded.
The statement related to the standard permutahedron holds based on the affine transformation
 between the standard permutahedron and $\cent{Q}_n$.
\qed
\end{proof}
%
\begin{algorithm}[t]
\caption{Projection onto $\cent{Q}_n$\label{alg:proj}}
\begin{algorithmic}[1]
\STATE 
 given $(q_1,\dots, q_n) \in [-1,1]^n$
 satisfying  
 $q_1\geq q_2\geq \dots \geq q_n$. \emph{(This assumption holds by renaming the indices, and reverting to their original names at the end).}\\
\STATE set $i_0=0$
\FOR {$k=1,\dots,n$}
\STATE for each $i=i_{k-1}+1,\dots,n$, set 
 $\delta_{i}^k = \min_{\delta \in \Real}\delta$ subject to:

\hspace{2cm}$\sum_{j=i_{k-1}+1}^i f^{-1}(f(q_j) -\delta) \leq \frac{2}{n-1}\sum_{j=i_{k-1}+1}^i \left(\frac{n+1}{2}-j\right).$
\STATE
$i_k=\arg\max_{i:i_{k-1}<i\leq n}\delta_{i}^k$.  
In case of  multiple minimizers, choose  largest  as $i_k$. 
 \STATE
 set $p_{j}=f^{-1}(f(q_{j}) - \delta_{i_k}^k)$ for $j=i_{k-1}+1,\dots,i_k$
\STATE {\bf if} $i_k=n$, {\bf then} break
\ENDFOR
 \RETURN $(p_1,\dots, p_n)'$
\end{algorithmic}
\end{algorithm}
%

\vspace{-2ex}
\section{Future Work}
The main open question is whether there is an algorithm of expected regret $O(n\sqrt{T})$ and time $O(n^3)$  in the setting of Section~\ref{sec:main}.
Another interesting line of research is to study other ranking polytopes.  For example, given any strictly monotonically  increasing function $f: \R \mapsto \R$
we can consider as an action set $f^n(S_n)$, defined as
$  f^n(S_n) := \{(f(\pi(1)), f(\pi(2)),\dots, f(\pi(n))):\, \pi \in S_n\}$.

\section*{Acknowledgments}
Ailon acknowledges the generous support of a Marie Curie Reintegration Grant PIRG07-GA-2010-268403, an Israel
Science Foundation (ISF) grant 127/133 and a Jacobs Technion-Cornell Innovation Institute (JTCII) grant.

\bibliography{online_AUC,myColt}
\bibliographystyle{plain}

\appendix

\section{Derivations in proof of Lemma~\ref{lem:main}}\label{sec:Hdetails}

By definition, and then by applying the properties of the distribution $\PL(w)$,
\begin{align}
H_{aa} &= \left[p(a\prec b \wedge a\prec c|w) + p(b\prec a\wedge c\prec a|w) - p(b\prec a\wedge a\prec c|w) - p(c\prec a\wedge a\prec b|w)\right ] \nonumber \\
&\hspace{0.2cm}- \left [ p(a\prec b|w)p(a\prec c|w) + p(b\prec a|w)p(c\prec a|w) - p(a\prec b|w)p(c\prec a|w) \right . \nonumber \\
&\left .\hspace{6cm} - p(a\prec c|w)p(b\prec a|w) \right ] \label{Haa}
\end{align}
\begin{align}
p(a\prec b \wedge a\prec c|w) &=\frac{e^{s(a)}}{e^{s(a)}+e^{s(b)}+e^{s(c)}}  \label{aaa1} \\
p(b\prec a\wedge c\prec a|w) &= \frac{e^{s(b)}} {e^{s(a)}+e^{s(b)}+e^{s(c)}}\frac{e^{s(c)}}{e^{s(a)}+e^{s(c)}}  +
\frac{e^{s(c)}} {e^{s(a)}+e^{s(b)}+e^{s(c)}}\frac{e^{s(b)}}{e^{s(a)}+e^{s(b)}} \label{aaa2} \\
p(b\prec a\wedge a\prec c|w) &= \frac{e^{s(b)}}{e^{s(a)}+e^{s(b)}+e^{s(c)}}\frac{e^{s(a)}}{e^{s(a)}+e^{s(c)}} \label{aaa3}\\
p(c\prec a\wedge a\prec b|w) &= \frac{e^{s(c)}}{e^{s(a)}+e^{s(b)}+e^{s(c)}}\frac{e^{s(a)}}{e^{s(a)}+e^{s(b)}} \label{aaa4}
\end{align}
Plugging  (\ref{aaa1})-(\ref{aaa4}) in (\ref{Haa}) and simplifying results in  (\ref{h1}).  
One now verifies:
\begin{align}
H_{ab} &= \left [ p(a\prec c\wedge b\prec c|w) + p(c\prec a\wedge c\prec b|w)  - 3p(a\prec c \wedge c\prec b|w)  - 3p(b\prec c\wedge c\prec a|w) \right] \nonumber \\
&\hspace{0.2cm}- \left [ -p(a\prec b|w)p(a\prec c|w) -p(b\prec a|w)p(c\prec a|w) +p(a\prec b|w)p(c\prec a|w)  \right .\nonumber \\
&\hspace{0.8cm}+ p(a\prec b|w)p(a\prec c|w) +p(a\prec b|w)p(b\prec c|w) + p(b\prec a|w)p(c\prec b|w) \nonumber \\  
&\hspace{0.8cm}- p(b\prec a|w)p(b\prec c|w) - p(a\prec b|w)p(c\prec b|w) +p(a\prec b|w)p(b\prec c|w) \nonumber \\ 
&\hspace{0.8cm}\left .+ p(b\prec a|w)p(c\prec b|w) - p(b\prec a|w)p(b\prec c|w) - p(a\prec b|w)p(c\prec b|w) \right . \nonumber \\
&\hspace{0.8cm} \left . -p(a\prec c|w)p(b\prec c|w) - p(c\prec a|w)p(c\prec b|w) + p(a\prec c|w)p(c\prec b|w) \right . \nonumber \\
&\hspace{9.5cm} \left .+ p(c\prec a|w)p(b\prec c|w) \right ] \nonumber 
\end{align}
Again using identities (\ref{aaa1})-(\ref{aaa4}) and simplifying, gives (\ref{h3})

\end{document}